\documentclass[11pt]{article}

\usepackage[preprint]{acl}

\usepackage{times}
\usepackage{latexsym}
\usepackage{comment}
\usepackage[T1]{fontenc}

\usepackage[utf8]{inputenc}

\usepackage{microtype}

\usepackage{inconsolata}

\usepackage{graphicx}
\graphicspath{{fig/}}

\usepackage{amsmath, amsfonts, amsthm}
\usepackage{booktabs}
\usepackage{array}
\usepackage{pdflscape}
\usepackage{tabularx}

\usepackage{xcolor}
\usepackage{listings}
\lstdefinestyle{leadimageprompt}{%
  basicstyle=\fontsize{6.4pt}{7.0pt}\selectfont\ttfamily,
  breaklines=true,
  breakatwhitespace=true,
  columns=fullflexible,
  keepspaces=true,
  showstringspaces=false,
  frame=single,
  framerule=0.2pt,
  rulecolor=\color{black!25},
  framesep=2pt,
  xleftmargin=0pt,
  xrightmargin=0pt,
  aboveskip=0.25em,
  belowskip=0.25em,
  lineskip=-0.3pt,
  tabsize=2
}

\newtheorem{proposition}{Proposition}

\theoremstyle{remark}

\newcommand{\E}{\mathbb{E}}
\newcommand{\1}{\mathbf{1}}

\DeclareMathOperator*{\argmax}{arg\,max}

\newcommand{\ruleimg}[1]{%
  \includegraphics[width=0.9\linewidth,height=0.48in,keepaspectratio]{#1}%
}

\newcommand{\ambigimg}[1]{%
  \begin{minipage}[c][0.72in][c]{\linewidth}
    \centering
    \includegraphics[width=0.95\linewidth,height=0.68in,keepaspectratio]{#1}
  \end{minipage}%
}

\title{The Differential Reasoning Router: Operationalizing Cost-Aware LLM Annotation in E-commerce}

\author{Cheng Lyu, \ Jingyue Zhang, \ Vinny DeGenova, \ Mengwei Li, \ and Yuanli Pei\\
  Wayfair\\
  \texttt{\{clyu1, jzhang2, vdegenova, mli2, ypei\}@wayfair.com}
}

\begin{document}
\maketitle

\begin{abstract}
Large Language Models (LLMs) are increasingly used to annotate structured product data in e-commerce, but early deployment often begins as a cold-start problem: only limited pre-launch labels are available, the value of expensive reasoning is unknown, and human review is needed before the system can be trusted at scale.
This challenge is especially common in rule-based annotation workflows, where each item must satisfy multiple business rules and both model errors and ambiguous rule boundaries affect final decisions.
We introduce the Differential Reasoning Router (DRR), a cost-aware framework for cold-start LLM annotation that jointly optimizes model selection and human escalation.
Rather than treating a reasoning model as a default fallback, DRR estimates separate success probabilities for a direct model and a reasoning model at both the sample and business-rule levels, enabling adaptive routing: easy cases are handled directly, reasoning is reserved for cases where it is expected to improve the decision, and likely double-failure or rule-disagreement cases are escalated to human annotators.
The resulting labels provide targeted ground truth for prompt engineering, supervised fine-tuning, calibration, and rule refinement, enabling a gradual shift from human-heavy cold-start annotation toward high-confidence automated routing.
In a production e-commerce workflow, DRR reaches accuracy parity with the strongest confidence-based router while achieving more than 60\% reasoning-token cost savings.
\end{abstract}
\section{Introduction}

In e-commerce, high-quality structured product data are essential for search, recommendation, merchandising, and compliance systems.
Maintaining these data often requires rule-based judgments over multimodal evidence, such as whether an image or product description supports a required attribute, whether the displayed product matches the listing, or whether the evidence is too ambiguous for an automated decision.
Human reviewers can make these judgments reliably, but catalog review at billion-product scale creates high cost, long latency, and delayed time-to-market for new or updated listings.

Large Language Models (LLMs) and autonomous agents have shown promise in automating parts of this workflow \citep{mohta2023large, calderon2025alternative}, and prior work has studied when LLM labels can substitute for or complement human annotation.
Existing approaches that select between direct and reasoning models rely heavily on human-labeled datasets.
However, new product attributes and evolving business definitions often require immediate catalog-wide deployment before a stable and representative human-annotated set can be constructed.
Moreover, human annotations are initially scarce and expensive to obtain.
Random sampling may fail to capture challenging or ambiguous cases, and early annotations can quickly become outdated as preliminary rule definitions are refined over time.
This creates a cold-start annotation problem: direct prompting is inexpensive but brittle, reasoning can help selectively at higher cost, and the system must allocate traffic across direct inference, additional reasoning, and human review while evidence about the task is still being accumulated.
In this context, cold start therefore refers not to the absence of labels, but to a mandatory yet limited pre-launch seed set that is small relative to catalog scale and still evolving with the business rules.

Existing work on model routing improves adaptive computation and test-time compute allocation \citep{wang2025reason, liang2025thinkswitcher, ding2025best, damani2024learning}, but cold-start annotation introduces two requirements that are not addressed by conventional routing methods alone.
First, the router must estimate the \emph{marginal value of reasoning}: whether a more expensive reasoning model is likely to change an otherwise direct decision enough to justify its token cost.
Scalar confidence and calibration methods can identify uncertain predictions \citep{kadavath2022language, ulmer2024calibrating, khanmohammadi2025calibrating}, but they do not directly distinguish examples where reasoning helps from examples where both automated models fail.
Second, the router must treat human review as part of the policy.
Reasoning models are not reliable universal fallbacks \citep{shojaee2025illusion, zhu2025can}, and some failures reflect ambiguous business rules or missing ground truth rather than insufficient model capacity.
In these cases, human review is not only a safety mechanism for the current decision, but also a way to acquire labels that reduce future uncertainty.

We propose the \textbf{Differential Reasoning Router (DRR)}, a cost-aware framework for cold-start LLM annotation.
DRR predicts separate success probabilities for a lightweight Direct model ($M_d$) and a higher-cost Reasoning model ($M_r$), both at the sample level and at the level of individual business rules.
These differential estimates allow the router to send easy cases to $M_d$, reserve $M_r$ for cases with positive expected reasoning value, and escalate likely double-failure or rule-disagreement cases to human reviewers.
As reviewed labels accumulate, they can be used for prompt iteration, supervised fine-tuning, calibration, and rule refinement, enabling the system to shift gradually from human-heavy cold-start operation toward higher-confidence automated routing.

Our contributions are threefold.
First, we formalize cold-start LLM annotation as a joint routing and label-acquisition problem in which the system chooses among direct inference, reasoning, and human review under limited ground truth.
Second, we introduce a differentially supervised, budget-constrained routing policy that estimates the marginal value of reasoning while detecting cases that require abstention.
Third, we evaluate DRR in a production e-commerce attribute-validation workflow, where it achieves accuracy parity with the strongest confidence-based router while yielding roughly one-fifth higher reasoning-token savings and identifying rule-level disagreement patterns that inform ground-truth collection and rule refinement.
\section{Related Work}

\paragraph{LLMs as Data Annotators.}
\citet{mohta2023large} benchmark LLMs as annotators across several tasks and show that downstream models trained on LLM-generated labels consistently underperform those trained on human labels.
Complementing this empirical view, \citet{calderon2025alternative} propose a statistical framework for determining when LLMs can credibly replace human annotators using a limited amount of human labeling.
This motivates deployment settings where human labels are scarce but strategically valuable, particularly when a new annotation task is still being calibrated.
Domain-specific studies, including financial-data annotation, further show that task expertise, prompt design, and review cost shape whether LLM labels can substitute for human judgments \citep{aguda2024large}.
Nevertheless, in domain-specific industrial deployments, hallucinations and reliability concerns often remain a key barrier to adoption.
Efforts to mitigate this via calibration and uncertainty estimation, such as probing whether models ``know what they know'' \citep{kadavath2022language} or learning confidence signals from model generations and stability under perturbations \citep{ulmer2024calibrating, khanmohammadi2025calibrating}, are widespread.

\paragraph{Adaptive Computation and Model Routing.}
\citet{wang2025reason} introduce a semantic router for vLLM that predicts a query's reasoning needs and routes it to a cheaper non-reasoning path when appropriate, reducing unnecessary test-time compute.
\citet{liang2025thinkswitcher} propose ThinkSwitcher, which dynamically switches a single model between short (fast) and long (slow) chain-of-thought modes based on input complexity.
\citet{ding2025best} present BEST-Route, an adaptive routing framework for test-time compute allocation that selects both the model and the best-of-$n$ sampling budget to meet a target quality at minimal estimated cost.
Related routing work also studies learning cost-quality trade-offs from data \citep{damani2024learning}.
While these methods primarily optimize general efficiency, they typically treat the stronger model (or more compute) as a reliable fallback rather than modeling human review as an action for both abstention and label acquisition.

\paragraph{Learning to Defer and Selective Prediction.}
Learning-to-defer and selective-prediction methods train systems to decide when to defer to an expert rather than make an automated prediction \citep{madras2018predict, mozannar2020consistent}.
These approaches typically study abstention for a fixed predictive model or a fixed human expert under accuracy, fairness, or workload objectives.
They do not directly address settings where abstention must be coordinated with model selection between direct and reasoning inference paths, or where deferred examples also serve as ground truth for prompt, rule, and policy refinement during cold-start deployment.

\paragraph{Limits of Reasoning Models.}
\citet{shojaee2025illusion} describe an ``illusion of thinking'', showing that inference-time reasoning effort (\textit{e.g.}, thinking-token compute) can exhibit diminishing or even non-monotonic returns as problem complexity increases.
In such cases, accuracy may stagnate or collapse even as models initially spend more compute.
\citet{zhu2025can} analyze behavioral divergence under ``save-thinking'' prompts, showing that reasoning models do not reliably suppress deliberation and can re-engage in explicit or implicit thinking depending on the instance.
These findings motivate routing policies that estimate the marginal value of reasoning over direct inference instead of assuming that harder examples should automatically receive more reasoning compute.
\section{Methodology}
\label{sec:methodology}

We formulate rule-based product annotation as cost-constrained multimodal routing.
Each query $q = (x_{\text{img}}, x_{\text{txt}})$ contains a product image and metadata, and must be evaluated against business rules $\mathcal{R} = \{r_1, \dots, r_k\}$, where $k$ is the total number of rules.
A product is valid only if all rules are satisfied.
DRR chooses among a low-cost Direct model ($M_d$), a higher-cost Reasoning model ($M_r$), and human review.
For each model $m \in \{d,r\}$ and rule $r_j$, $y_{m,j}=1$ means $M_m$ agrees with human ground truth on that rule, while $y_m^{\text{sys}}=1$ means it is correct on the full rule set.

The policy $\pi(q) \in \{M_d, M_r, \text{Human}\}$ must choose among three possible actions:  the direct model is sufficient, the reasoning model is worth its extra cost, or human review is needed.
Let $C_d(q)$ and $C_r(q)$ denote the model costs, and define the incremental reasoning cost $\Delta C_q = C_r(q)-C_d(q) \ge 0$.
The ideal policy minimizes expected error subject to an average reasoning budget $B_{\text{target}}$ and a human-referral budget $H_{\text{target}}$:
\begin{equation}
\begin{aligned}
    \min_{\pi} \quad & \E_q[\text{Risk}(\pi(q),q)] \\
    \text{s.t.} \quad & \E_q[\Delta C_q \cdot \1\{\pi(q)=M_r\}] \le B_{\text{target}}, \\
    & \E_q[\1\{\pi(q)=\text{Human}\}] \le H_{\text{target}}.
\end{aligned}
\label{eq:routing_objective}
\end{equation}
We enforce the reasoning budget during training with a nonnegative budget multiplier, and handle human referrals with validation-tuned thresholds.

\begin{figure}[t]
    \centering
    \includegraphics[width=\columnwidth]{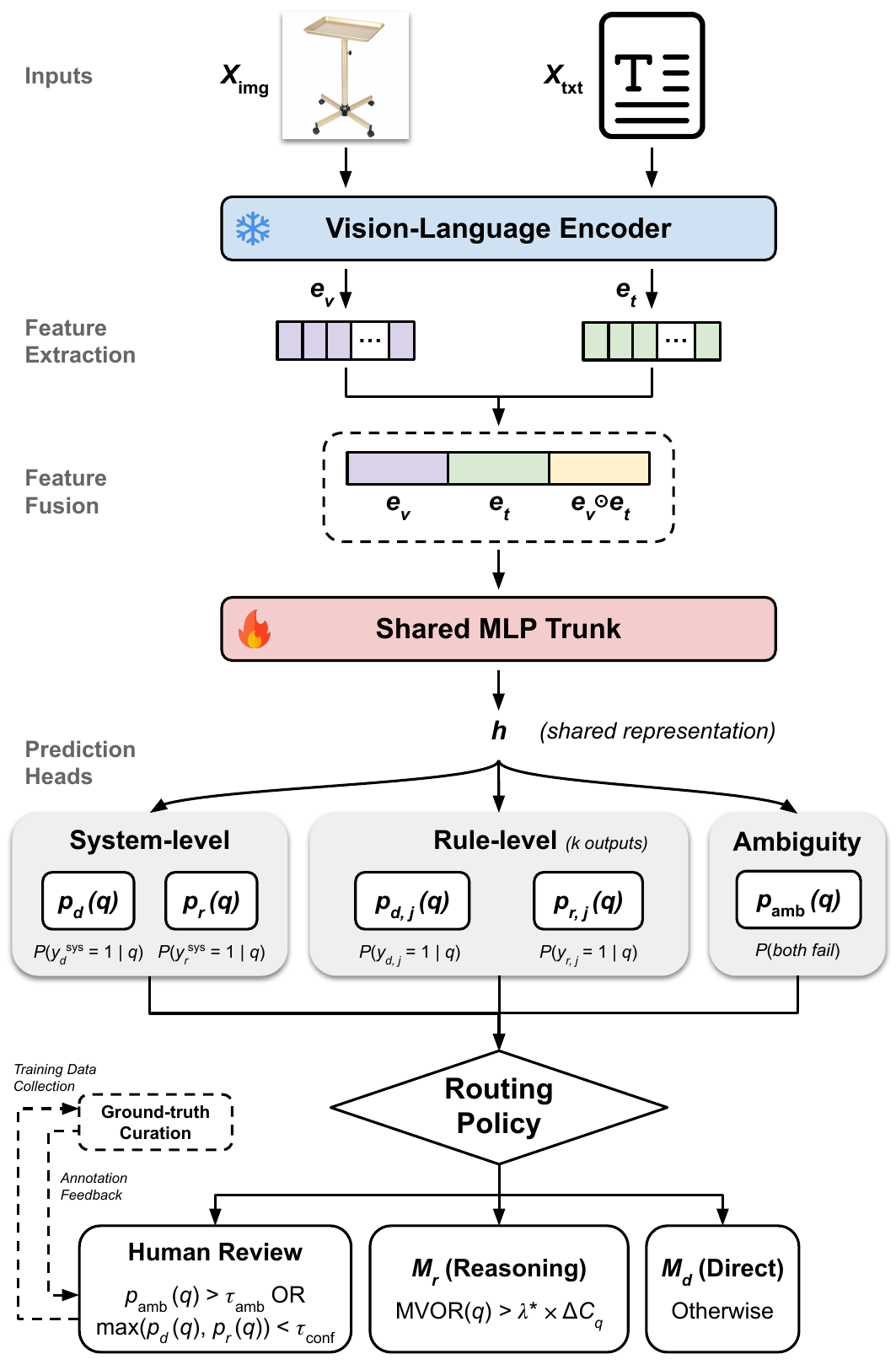}
    \caption{Overall DRR architecture for cost-aware three-way routing with human-feedback support.}
    \label{fig:architecture}
\end{figure}

Figure~\ref{fig:architecture} summarizes the design.
Expensive multimodal features are computed once and cached; online routing uses only a small supervised model, which keeps the decision layer compatible with production latency.

\subsection{Architecture: Lightweight Differential Predictor}

\paragraph{Feature extraction and fusion.}
We use a pretrained multimodal encoder to obtain image embeddings $e_v$ and text embeddings $e_t$.
The router does not fine-tune this encoder; it consumes cached features, which keeps routing fast and stable.
To capture both unimodal evidence and image-text alignment, we concatenate image, text, and interaction features, $z = [e_v; e_t; e_v \odot e_t]$, then pass $z$ through a shared MLP trunk.

\paragraph{Prediction heads.}
DRR uses three types of heads.
System-level heads predict $\hat{p}_d(q)=P(y_d^{\text{sys}}=1\mid q)$ and $\hat{p}_r(q)=P(y_r^{\text{sys}}=1\mid q)$, the probabilities that each automated model succeeds on the full rule set.
Rule-level heads predict $\hat{p}_{d,j}(q)=P(y_{d,j}=1\mid q)$ and $\hat{p}_{r,j}(q)=P(y_{r,j}=1\mid q)$, the probabilities that $M_d$ and $M_r$ agree with human ground truth on rule $r_j$.
An ambiguity head predicts $\hat{p}_{\text{amb}}(q)$, the probability that both automated models fail, so the router can escalate cases where more reasoning is unlikely to help.

\subsection{Training Objective}

DRR is trained to estimate annotator reliability and to spend reasoning budget only where it is expected to improve the decision.
For a fixed budget multiplier $\lambda$, we combine these goals in an objective minimized over the router parameters $\theta$:

\begin{equation}
\begin{aligned}
    \mathcal{J}(\theta, \lambda) & = \mathcal{L}_{\text{sup}}(\theta) + \mathcal{L}_{\text{route}}(\theta) \\
    & + \lambda(\bar{C}_{\text{route}} - B_{\text{target}}).
\end{aligned}
\label{eq:training_objective}
\end{equation}
Here $\mathcal{L}_{\text{sup}}$ trains the system-level, rule-level, and ambiguity heads with binary cross-entropy losses; $\mathcal{L}_{\text{route}}$ is the expected error of the relaxed routing policy; and $\bar{C}_{\text{route}}$ is the expected incremental reasoning cost.
The multiplier $\lambda$ acts as the learned price of spending reasoning budget: larger values make routing to $M_r$ harder to justify.

The supervision is differential rather than distillational: DRR learns whether each model matches human ground truth, not whether $M_d$ should imitate $M_r$.
This distinction is important because reasoning can have negative value when $M_r$ fails on a query while $M_d$ gets it correct.
The ambiguity target is positive when both system-level correctness labels are zero, i.e., $y_d^{\text{sys}}=0$ and $y_r^{\text{sys}}=0$; we use a weighted binary cross-entropy loss for this head because double failures are less common but operationally important.

The routing loss uses the system-level success estimates to compare the two automated actions.
For a query that is not sent to human review, the estimated direct and reasoning risks are $1-\hat{p}_d(q)$ and $1-\hat{p}_r(q)$, respectively.
The central routing quantity is therefore the \textbf{Marginal Value of Reasoning} (MVOR), the predicted reduction in risk from using $M_r$ instead of $M_d$:
\begin{equation*}
    \text{MVOR}(q) = \hat{p}_r(q) - \hat{p}_d(q).
\end{equation*}
A positive MVOR means reasoning is expected to improve correctness; a negative MVOR means the direct model is expected to be better.
During training, the binary choice between $M_d$ and $M_r$ is relaxed with a sigmoid policy so that the routing loss remains differentiable:
\begin{equation*}
    \pi_r(q) = \sigma\left(\frac{\text{MVOR}(q) - \lambda \Delta C_q}{T}\right),
\end{equation*}
where $T$ is a temperature.
The term $\lambda \Delta C_q$ is the opportunity cost of spending reasoning compute on query $q$.

The relaxed policy then determines both terms in Equation~\ref{eq:training_objective}: $\mathcal{L}_{\text{route}}$ averages the expected automated error under this mixture of $M_d$ and $M_r$, while $\bar{C}_{\text{route}}$ averages the corresponding incremental reasoning cost.
Thus Equation~\ref{eq:training_objective} connects probability learning with budgeted routing: the supervised term learns which model is likely to be correct, while the routing and Lagrangian terms learn how selectively to spend reasoning tokens.

\subsection{Optimization}

We optimize Equation~\ref{eq:training_objective} with alternating primal and dual updates.
The model parameters $\theta$ are updated by gradient descent, while the multiplier is updated by gradient ascent on the budget violation.
With $\lambda=e^{\nu}$, the log-space update is
\begin{equation*}
    \nu \leftarrow \nu + \eta_\lambda (\bar{C}_{\text{route}} - B_{\text{target}}).
\end{equation*}
If expected reasoning cost exceeds the budget, $\lambda$ increases and the router becomes more selective; if the model is under budget, $\lambda$ decreases and more reasoning calls are allowed.
Appendix~\ref{subsec:convergence_proof} gives the budget-response and positivity properties of this log-space update, which justify its use as a stable budget price during training.

\subsection{Inference Policy}

At inference time, DRR uses the learned probabilities and the converged budget multiplier $\lambda^*$.
It first applies the human gate:
\begin{equation*}
    \hat{p}_{\text{amb}}(q) > \tau_{\text{amb}}
    \quad \text{or} \quad
    \max(\hat{p}_d(q),\hat{p}_r(q)) < \tau_{\text{conf}}.
\end{equation*}
If either condition holds, the query is routed to human review, covering both predicted double failures and broadly uncertain cases.

For the remaining automated cases, the model-selection decision is
\begin{equation}
    \pi(q) =
    \begin{cases}
        M_r, & \text{if } \text{MVOR}(q) > \lambda^* \Delta C_q, \\
        M_d, & \text{otherwise.}
    \end{cases}
\end{equation}
Appendix~\ref{subsec:threshold_optimality} derives this cost-adjusted threshold rule for the automated cases that pass the human gate.
The thresholds $\tau_{\text{amb}}$ and $\tau_{\text{conf}}$ are tuned on validation data under a human-referral budget.
This gives operators two independent controls: $B_{\text{target}}$ governs reasoning spend, while the escalation thresholds govern human review.

\subsection{Cold-start Controls and Feedback Loop}
\label{subsec:cold_start}

The policy is designed to support cold-start deployment and continuous improvement.
When task-specific labels are scarce, DRR can initially use conservative human-escalation thresholds to maintain catalog quality while routing uncertain or likely double-failure cases to reviewers.
These escalations form an informative label stream rather than a purely random sample: they concentrate on ambiguous examples, cases where $M_d$ and $M_r$ disagree, and instances where both automated paths appear unreliable.
As review labels accumulate, DRR can be retrained or recalibrated, prompts or supervised fine-tuning data can be updated, and thresholds can be relaxed to reduce human traffic while retaining an abstention pathway for unreliable cases.
Rule-level heads further reveal which business rules underlie disagreement, creating a feedback loop for ground-truth collection and rule refinement.
\section{Experiments}
\label{sec:experiments}

\subsection{Setup}

\paragraph{Task.}
We evaluate on a production lead-image eligibility task at a major e-commerce platform.
Each sample contains a product image and metadata, and the system must decide whether the image can serve as the primary product image.
Human ground truth evaluates each sample against $k{=}11$ conjunctive business rules spanning objective checks (e.g., no overlaid text, no visible person) and subjective visual judgments (e.g., product is the primary focus, front of product is sufficiently visible).
A single rule failure makes the image invalid, so the task is sensitive to both model errors and ambiguous rule interpretation.
Appendix~\ref{subsec:lead_image_rules} provides the complete rule definitions, and Appendix~\ref{subsec:subjective_rule_examples} shows borderline subjective examples.
A sample is \emph{unsolvable} when both $M_d$ and $M_r$ fail on the full rule set, which motivates DRR's human-review action.

\paragraph{Data split and cold-start protocol.}
The labeled corpus contains 9{,}358 examples, split into 6{,}550 training, 1{,}403 validation, and 1{,}405 test examples.
These labels come from the mandatory pre-launch review process used for launch approval, model/vendor selection, prompt tuning, and operating-point selection; they were not collected as a separate production annotation campaign solely for DRR.
This is the sense in which the deployment is cold-start: labels exist, but they are limited relative to catalog scale and still evolving with the business rules.
DRR uses the training split to learn task-specific reliability heads, while both DRR and MaxConf use the same validation split to select human-review thresholds.
We therefore interpret the comparison as an operational use of required pre-production evidence, not as a label-efficiency claim.

\paragraph{Models.}
Both automated annotators use Gemini~2.5 Flash Lite with the same structured output prompts; Appendix~\ref{app:lead-image-prompt} provides the full annotation prompts.
The Direct model ($M_d$) runs with thinking disabled, while the Reasoning model ($M_r$) runs with dynamic thinking enabled, allowing the model to spend additional reasoning tokens when invoked.
Note that DRR is not tied to this model version: the router only requires candidate model outputs, realized costs, and ground-truth agreement labels, so newer direct or reasoning models can be substituted by regenerating candidate predictions and recalibrating or retraining the routing layer.
We use OpenCLIP ViT-L/14 for feature extraction, with a shared MLP trunk.
Embeddings are precomputed and cached. 

\paragraph{Baselines.}
The baselines isolate the design choices in Section~\ref{sec:methodology} using the same model outputs.
Direct-only always uses $M_d$ and represents the cheapest automated deployment.
Reasoning-only always uses $M_r$ and serves as the reasoning-token cost baseline.
Random routes uniformly between $M_d$ and $M_r$, testing whether learned routing adds value beyond arbitrary model selection.
Oracle is a best-of-two automated selector that chooses a correct automated model whenever either $M_d$ or $M_r$ is correct; it shows the limit of automated selection without human review.
MaxConf follows confidence-based query-aware LLM selection baselines \citep{maurya2025selectllm}: it routes a sample to human review when the maximum model-reported Direct/Reasoning confidence is below a threshold, and otherwise selects $M_r$ only when its confidence exceeds $M_d$'s.
Reasoning-token savings for MaxConf count the incremental reasoning cost only on samples where the policy selects $M_r$, using the same accounting as DRR.
MVOR-only uses the cost-aware $M_d$/$M_r$ rule without human escalation, isolating the need for abstention.
For Pareto sweeps, Ambiguity-only uses only the learned double-failure score for human escalation.
Thresholded methods are tuned on the validation set under a ${\leq}$20\% human-referral budget.

\paragraph{Metrics.}
\emph{End-to-end accuracy} measures whether the final routed decision is correct: automated routes are correct when the selected model matches human ground truth on the full rule set, and human-reviewed samples are counted as correct.
\emph{Human review rate} reports the fraction of samples routed to reviewers.
\emph{Reasoning token savings} measures reduction in reasoning-token spend relative to always using $M_r$; it does not include human-review labor.
For diagnostic analyses, we distinguish two granularities of double failure: an \emph{unsolvable instance} is a full sample where both $M_d$ and $M_r$ fail the conjunctive decision, while a \emph{rule-level joint failure} is a specific business rule on which both models disagree with human ground truth.

\subsection{Main Results}

\begin{table}[t]
\centering
\resizebox{\columnwidth}{!}{%
\begin{tabular}{lccc}
\toprule
\textbf{Method} & \textbf{Acc (\%)} & \shortstack{\textbf{Human}\\\textbf{Review (\%)}} & \shortstack{\textbf{Reasoning-Token}\\\textbf{Savings (\%)}} \\
\midrule
Direct-only ($M_d$) & 68.0 \scriptsize{[65.6, 70.5]} & --- & 100.0 \\
Reasoning-only ($M_r$) & 69.3 \scriptsize{[66.8, 71.6]} & --- & 0.0 \\
Random & 68.6 \scriptsize{[66.5, 70.7]} & --- & 50.0 \\
\midrule
MaxConf & 82.0 \scriptsize{[80.0, 84.0]} & 20.0 & 55.2 \\
MVOR-only & 69.5 \scriptsize{[67.1, 71.9]} & --- & 56.4 \\
\midrule
Oracle & 79.1 \scriptsize{[77.0, 81.1]} & --- & 86.6 \\
\textbf{DRR (ours)} & \textbf{82.8} \scriptsize{[80.8, 84.8]} & 21.7 & \textbf{66.2} \\
\bottomrule
\end{tabular}%
}
\caption{Main results on the product attribute test set ($N{=}1{,}405$). Acc is end-to-end accuracy after routing. Brackets are 95\% bootstrap confidence intervals. Reasoning-token savings are measured relative to always using $M_r$. Thresholded methods are selected on validation set; the table reports realized test-set review rates.}
\label{tab:main_results}
\end{table}

Table~\ref{tab:main_results} addresses the practical question behind the experiment: after allocating limited human review, how much reasoning compute is still needed.
The single-model and random rows show that invoking or randomly selecting the reasoning model changes accuracy only modestly.
The policies that can route difficult cases to humans reach a higher accuracy range.
The discussion below separates these effects into human review, cost-aware reasoning, and ambiguity supervision.

\paragraph{Human review must be a routing action.}
The methods that always choose between $M_d$ and $M_r$ top out at 79.1\% accuracy with Oracle, below the thresholded methods that can route examples to human review.
This pattern is expected when some samples are unsolvable by both candidate models: choosing the better automated output is still insufficient for those cases.
Human review therefore needs to be modeled as a routing action, rather than added only after automated model selection.

\paragraph{Cost-aware routing must be paired with abstention.}
MVOR-only saves reasoning tokens but remains close to the single-model baselines in accuracy.
This separates two effects: the MVOR threshold controls when to spend reasoning tokens, but it does not by itself identify cases that should leave the automated path.
After adding the ambiguity gate, DRR reaches the accuracy range of the methods that route examples to humans while achieving more than 60\% reasoning-token cost savings compared to MaxConf.
Thus cost-aware model selection is most useful when paired with abstention for likely double failures.

\paragraph{Ambiguity supervision adds signal beyond confidence.}
MaxConf performs well because model-reported confidence is a useful first filter: many difficult examples are also low-confidence examples.
However, for cold-start annotation, the review set is also training signal for the next iteration, so the useful cases are not necessarily the least confident ones.
The more valuable cases are those that reveal where automation fails, especially examples likely to fail under both $M_d$ and $M_r$ or to expose unstable rule boundaries.
DRR uses differential supervision to learn direct-model success, reasoning-model success, and joint automated failure, giving the human gate a task-specific signal beyond confidence alone.

\subsection{Routing Zone Analysis}
\label{subsec:zone_analysis}

\begin{table}[t]
\centering
\resizebox{\columnwidth}{!}{%
\begin{tabular}{lrrrrr}
\toprule
\textbf{Zone} & $N$ \textbf{Inst.} & \textbf{\% Inst.} & \textbf{$M_d$ Acc} & \textbf{$M_r$ Acc} & \textbf{\% Unsolv. Inst.} \\
\midrule
Economy & 570 & 40.6 & 76.7 & 78.6 & 12.8 \\
Reasoning & 530 & 37.7 & 75.1 & 79.6 & 12.5 \\
Ambiguity & 305 & 21.7 & 39.7 & 33.8 & \textbf{50.8} \\
\midrule
Population & 1405 & 100 & 68.0 & 69.3 & 20.9 \\
\bottomrule
\end{tabular}%
}
\caption{Per-zone statistics on the test set. Zones correspond to DRR's three inference actions: direct routing, reasoning routing, and human escalation. Unsolv. Inst. is the percentage of instances in each zone where both automated models fail the full rule set.}
\label{tab:zone_analysis}
\end{table}

Table~\ref{tab:zone_analysis} checks whether the inference policy creates meaningful regions.
The Ambiguity Zone has the highest unsolvable-instance rate, while the Economy and Reasoning Zones retain mostly automatable samples.
Intuitively, Economy examples tend to have sufficient direct evidence, Reasoning examples require additional cross-checking of the image and metadata, and Ambiguity examples resemble the borderline subjective cases in Appendix~\ref{subsec:subjective_rule_examples}.
This validates DRR's ability to preserve human capacity while choosing between cheap and expensive model calls.

\subsection{Per-Rule Diagnostics}
\label{subsec:rule_diagnostics}

\begin{table}[t]
\centering
\resizebox{\columnwidth}{!}{%
\begin{tabular}{lcccl}
\toprule
\textbf{Rule} & $M_d$ \textbf{Acc}  & $M_r$ \textbf{Acc}  & \textbf{\% Joint Fail} & \textbf{Type} \\
\midrule
D1: Primary Focus & 68.1 & 79.4 & 17.4 & Subj. \\
E1: Front Visible & 76.7 & 78.5 & 16.7 & Subj. \\
D5: Photorealistic & 80.6 & 81.9 & 14.5 & Subj. \\
D3: Standard View & 79.3 & 80.1 & 14.0 & Subj. \\
E3: Correct Set & 90.1 & 84.3 & 6.3 & Subj. \\
\midrule
D2: Functionality & 95.9 & 91.0 & 0.0 & Obj. \\
E2: Standard State & 95.2 & 94.2 & 0.0 & Obj. \\
E4: Props & 96.9 & 96.7 & 0.0 & Obj. \\
D4: Dimensions & 98.7 & 98.0 & 0.0 & Obj. \\
E5: No Text & 97.7 & 99.0 & 0.0 & Obj. \\
D6: No Person & 99.0 & 99.3 & 0.0 & Obj. \\
\bottomrule
\end{tabular}%
}
\caption{Per-rule diagnostics. Joint Fail is the percentage of samples for which both automated models fail the same rule. Subjective rules from Appendix~\ref{subsec:subjective_rule_examples} concentrate these rule-level joint failures, while objective checks are more stable.}
\label{tab:rule_analysis}
\end{table}

Table~\ref{tab:rule_analysis} explains why the ambiguity region exists: rule-level joint failures concentrate in subjective rules, matching the borderline cases in Appendix~\ref{subsec:subjective_rule_examples}.
Objective checks are more stable because their criteria are directly verifiable.
This validates the rule-level heads in Section~\ref{sec:methodology}: they turn system-level routing failures into diagnoses of which business rules need clarification or human oversight.
Human referrals therefore serve both as a quality safeguard and as targeted ground truth for future router training and rule refinement.

\subsection{Pareto Analysis}
\label{subsec:pareto}

\begin{figure}[t]
    \centering
    \includegraphics[width=\columnwidth]{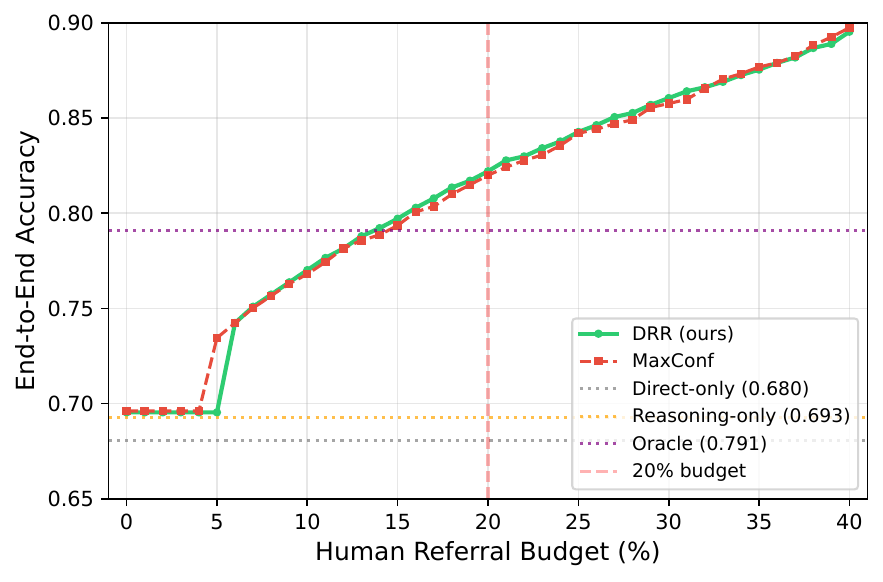}
    \caption{Accuracy-human budget trade-off. The x-axis varies the allowed human-referral budget, and the dashed line marks the operating point used in Table~\ref{tab:main_results}. Dotted horizontal lines show automated reference points. The curves show how different escalation policies convert limited human review into end-to-end accuracy during cold-start rollout.}
    \label{fig:pareto}
\end{figure}

Figure~\ref{fig:pareto} treats the human-referral threshold as a deployment control for cold-start annotation.
Moving right allocates more reviewer capacity while rules, prompts, and router calibration are still being refined; moving left preserves human capacity but requires more automated decisions.
The steep accuracy gains at low review budgets show why selective human escalation is valuable during rollout, and the operating region in Table~\ref{tab:main_results} shows that DRR remains competitive with simple confidence-based escalation while retaining explicit control over reasoning-token spend.
\section{Conclusions}
\label{sec:conclusions}

DRR shows that routing can do more than select the cheapest successful inference path.
By estimating where direct inference is sufficient, where reasoning is likely to help, and where both automated paths are unreliable, the framework turns cold-start LLM annotation into an explicit operating policy.
Its routing zones and rule-level diagnostics identify which traffic can be automated, which cases deserve reasoning tokens, and which examples should become human-labeled ground truth.
Rather than waiting for a large balanced evaluation set or committing to a single model choice, practitioners can launch with a controlled human-review budget, monitor the accuracy-review tradeoff, and improve prompts, calibration, training data, and rule definitions as the task warms up.

\section*{Limitations}
\label{sec:limitations}

This study evaluates DRR on a specific production e-commerce annotation workflow with rule based annotations.
Although the framework is intended for rule-based LLM annotations more broadly, its behavior may differ across domains, modalities, rule structures, and operational review processes.

Our experiments also reflect the model landscape at the time of deployment. Although DRR is model-agnostic by design, commercial LLMs evolve rapidly. For each new pair of models, we recommend regenerating candidate predictions and recalibrating or retraining the router to maximize performance gains.

We do not evaluate label efficiency across training-label budgets.
The reported DRR result uses the full pre-launch training split, while confidence thresholding requires only validation-set operating-point selection.
Thus, our comparison should be read as how to use a mandatory pre-production labeled corpus in deployment, not as evidence that DRR is preferable when all labels must be newly acquired.

Finally, we treat human review as ground truth, matching its role in the production decision process.
In practice, human labels can contain disagreement or inconsistency, especially for subjective rules.
Such noise may affect both reported accuracy and the supervision used to train ambiguity and rule-level failure signals.

\bibliography{main}

@misc{wang2025reason,
      title={When to Reason: Semantic Router for vLLM}, 
      author={Chen Wang and Xunzhuo Liu and Yuhan Liu and Yue Zhu and Xiangxi Mo and Junchen Jiang and Huamin Chen},
      year={2025},
      eprint={2510.08731},
      archivePrefix={arXiv},
      primaryClass={cs.ET},
      url={https://arxiv.org/abs/2510.08731}, 
}

@inproceedings{ding2025best,
author = {Ding, Dujian and Mallick, Ankur and Zhang, Shaokun and Wang, Chi and Madrigal, Daniel and Garcia, Mirian Del Carmen Hipolito and Xia, Menglin and Lakshmanan, Laks V. S. and Wu, Qingyun and R{\"u}hle, Victor},
title = {BEST-route: adaptive LLM routing with test-time optimal compute},
year = {2025},
publisher = {JMLR.org},
booktitle = {Proceedings of the 42nd International Conference on Machine Learning},
articleno = {530},
numpages = {15},
location = {Vancouver, Canada},
series = {ICML'25}
}

@inproceedings{damani2024learning,
 author = {Damani, Mehul and Shenfeld, Idan and Peng, Andi and Bobu, Andreea and Andreas, Jacob},
 booktitle = {International Conference on Learning Representations},
 editor = {Y. Yue and A. Garg and N. Peng and F. Sha and R. Yu},
 pages = {102783--102802},
 title = {Learning How Hard to Think: Input-Adaptive Allocation of LM Computation},
 url = {https://proceedings.iclr.cc/paper_files/paper/2025/file/ff414825df833edb8b1839e3d5d495e9-Paper-Conference.pdf},
 volume = {2025},
 year = {2025}
}

@misc{zhu2025can,
      title={When Can Large Reasoning Models Save Thinking? Mechanistic Analysis of Behavioral Divergence in Reasoning}, 
      author={Rongzhi Zhu and Yi Liu and Zequn Sun and Yiwei Wang and Wei Hu},
      year={2025},
      eprint={2505.15276},
      archivePrefix={arXiv},
      primaryClass={cs.AI},
      url={https://arxiv.org/abs/2505.15276}, 
}

@inproceedings{liang2025thinkswitcher,
    title = "{T}hink{S}witcher: When to Think Hard, When to Think Fast",
    author = "Liang, Guosheng  and
      Zhong, Longguang  and
      Yang, Ziyi  and
      Quan, Xiaojun",
    editor = "Christodoulopoulos, Christos  and
      Chakraborty, Tanmoy  and
      Rose, Carolyn  and
      Peng, Violet",
    booktitle = "Findings of the Association for Computational Linguistics: EMNLP 2025",
    month = nov,
    year = "2025",
    address = "Suzhou, China",
    publisher = "Association for Computational Linguistics",
    url = "https://aclanthology.org/2025.findings-emnlp.278/",
    doi = "10.18653/v1/2025.findings-emnlp.278",
    pages = "5185--5201",
    ISBN = "979-8-89176-335-7"
}

@inproceedings{shojaee2025illusion,
 author = {Shojaee, Parshin and Mirzadeh, Iman and Alizadeh vahid, Keivan and Horton, Maxwell and Bengio, Samy and Farajtabar, Mehrdad},
 booktitle = {Advances in Neural Information Processing Systems},
 doi = {10.52202/085713-3603},
 editor = {D. Belgrave and C. Zhang and H. Lin and R. Pascanu and P. Koniusz and M. Ghassemi and N. Chen},
 pages = {108018--108059},
 publisher = {Curran Associates, Inc.},
 title = {The Illusion of Thinking: Understanding the Strengths and Limitations of Reasoning Models via the Lens of Problem Complexity},
 url = {https://proceedings.neurips.cc/paper_files/paper/2025/file/9b26ad15462c81548c0689188d2e8018-Paper-Conference.pdf},
 volume = {38, Main Conference},
 year = {2025}
}

@InProceedings{mohta2023large,
  title = 	 {Are large language models good annotators?},
  author =       {Mohta, Jay and Ak, Kenan and Xu, Yan and Shen, Mingwei},
  booktitle = 	 {Proceedings on "I Can't Believe It's Not Better: Failure  Modes in the Age of Foundation Models" at NeurIPS 2023 Workshops},
  pages = 	 {38--48},
  year = 	 {2023},
  editor = 	 {Antorán, Javier and Blaas, Arno and Buchanan, Kelly and Feng, Fan and Fortuin, Vincent and Ghalebikesabi, Sahra and Kriegler, Andreas and Mason, Ian and Rohde, David and Ruiz, Francisco J. R. and Uelwer, Tobias and Xie, Yubin and Yang, Rui},
  volume = 	 {239},
  series = 	 {Proceedings of Machine Learning Research},
  month = 	 {16 Dec},
  publisher =    {PMLR},
  url = 	 {https://proceedings.mlr.press/v239/mohta23a.html},
}

@inproceedings{calderon2025alternative,
    title = "The Alternative Annotator Test for {LLM}-as-a-Judge: How to Statistically Justify Replacing Human Annotators with {LLM}s",
    author = "Calderon, Nitay  and
      Reichart, Roi  and
      Dror, Rotem",
    editor = "Che, Wanxiang  and
      Nabende, Joyce  and
      Shutova, Ekaterina  and
      Pilehvar, Mohammad Taher",
    booktitle = "Proceedings of the 63rd Annual Meeting of the Association for Computational Linguistics (Volume 1: Long Papers)",
    month = jul,
    year = "2025",
    address = "Vienna, Austria",
    publisher = "Association for Computational Linguistics",
    url = "https://aclanthology.org/2025.acl-long.782/",
    doi = "10.18653/v1/2025.acl-long.782",
    pages = "16051--16081",
    ISBN = "979-8-89176-251-0",
}

@misc{kadavath2022language,
      title={Language Models (Mostly) Know What They Know}, 
      author={Saurav Kadavath and Tom Conerly and Amanda Askell and Tom Henighan and Dawn Drain and Ethan Perez and Nicholas Schiefer and Zac Hatfield-Dodds and Nova DasSarma and Eli Tran-Johnson and Scott Johnston and Sheer El-Showk and Andy Jones and Nelson Elhage and Tristan Hume and Anna Chen and Yuntao Bai and Sam Bowman and Stanislav Fort and Deep Ganguli and Danny Hernandez and Josh Jacobson and Jackson Kernion and Shauna Kravec and Liane Lovitt and Kamal Ndousse and Catherine Olsson and Sam Ringer and Dario Amodei and Tom Brown and Jack Clark and Nicholas Joseph and Ben Mann and Sam McCandlish and Chris Olah and Jared Kaplan},
      year={2022},
      eprint={2207.05221},
      archivePrefix={arXiv},
      primaryClass={cs.CL},
      url={https://arxiv.org/abs/2207.05221}, 
}

@inproceedings{ulmer2024calibrating,
    title = "Calibrating Large Language Models Using Their Generations Only",
    author = "Ulmer, Dennis  and
      Gubri, Martin  and
      Lee, Hwaran  and
      Yun, Sangdoo  and
      Oh, Seong",
    editor = "Ku, Lun-Wei  and
      Martins, Andre  and
      Srikumar, Vivek",
    booktitle = "Proceedings of the 62nd Annual Meeting of the Association for Computational Linguistics (Volume 1: Long Papers)",
    month = aug,
    year = "2024",
    address = "Bangkok, Thailand",
    publisher = "Association for Computational Linguistics",
    url = "https://aclanthology.org/2024.acl-long.824/",
    doi = "10.18653/v1/2024.acl-long.824",
    pages = "15440--15459",
}

@inproceedings{khanmohammadi2025calibrating,
    title = "Calibrating {LLM} Confidence by Probing Perturbed Representation Stability",
    author = "Khanmohammadi, Reza  and
      Miahi, Erfan  and
      Mardikoraem, Mehrsa  and
      Kaur, Simerjot  and
      Brugere, Ivan  and
      Smiley, Charese  and
      Thind, Kundan S  and
      Ghassemi, Mohammad M.",
    editor = "Christodoulopoulos, Christos  and
      Chakraborty, Tanmoy  and
      Rose, Carolyn  and
      Peng, Violet",
    booktitle = "Proceedings of the 2025 Conference on Empirical Methods in Natural Language Processing",
    month = nov,
    year = "2025",
    address = "Suzhou, China",
    publisher = "Association for Computational Linguistics",
    url = "https://aclanthology.org/2025.emnlp-main.530/",
    doi = "10.18653/v1/2025.emnlp-main.530",
    pages = "10448--10514",
    ISBN = "979-8-89176-332-6",
}

@inproceedings{maurya2025selectllm,
    title = "{S}elect{LLM}: Query-Aware Efficient Selection Algorithm for Large Language Models",
    author = "Maurya, Kaushal Kumar  and
      Srivatsa, Kv Aditya  and
      Kochmar, Ekaterina",
    editor = "Che, Wanxiang  and
      Nabende, Joyce  and
      Shutova, Ekaterina  and
      Pilehvar, Mohammad Taher",
    booktitle = "Findings of the Association for Computational Linguistics: ACL 2025",
    month = jul,
    year = "2025",
    address = "Vienna, Austria",
    publisher = "Association for Computational Linguistics",
    url = "https://aclanthology.org/2025.findings-acl.1072/",
    doi = "10.18653/v1/2025.findings-acl.1072",
    pages = "20847--20863",
    ISBN = "979-8-89176-256-5",
}

@inproceedings{madras2018predict,
 author = {Madras, David and Pitassi, Toni and Zemel, Richard},
 booktitle = {Advances in Neural Information Processing Systems},
 editor = {S. Bengio and H. Wallach and H. Larochelle and K. Grauman and N. Cesa-Bianchi and R. Garnett},
 pages = {},
 publisher = {Curran Associates, Inc.},
 title = {Predict Responsibly: Improving Fairness and Accuracy by Learning to Defer},
 url = {https://proceedings.neurips.cc/paper_files/paper/2018/file/09d37c08f7b129e96277388757530c72-Paper.pdf},
 volume = {31},
 year = {2018}
}

@InProceedings{mozannar2020consistent,
  title = 	 {Consistent Estimators for Learning to Defer to an Expert},
  author =       {Mozannar, Hussein and Sontag, David},
  booktitle = 	 {Proceedings of the 37th International Conference on Machine Learning},
  pages = 	 {7076--7087},
  year = 	 {2020},
  editor = 	 {III, Hal Daumé and Singh, Aarti},
  volume = 	 {119},
  series = 	 {Proceedings of Machine Learning Research},
  month = 	 {13--18 Jul},
  publisher =    {PMLR},
  url = 	 {https://proceedings.mlr.press/v119/mozannar20b.html},
}

@inproceedings{aguda2024large,
    title = "Large Language Models as Financial Data Annotators: A Study on Effectiveness and Efficiency",
    author = "Aguda, Toyin D.  and
      Siddagangappa, Suchetha  and
      Kochkina, Elena  and
      Kaur, Simerjot  and
      Wang, Dongsheng  and
      Smiley, Charese",
    editor = "Calzolari, Nicoletta  and
      Kan, Min-Yen  and
      Hoste, Veronique  and
      Lenci, Alessandro  and
      Sakti, Sakriani  and
      Xue, Nianwen",
    booktitle = "Proceedings of the 2024 Joint International Conference on Computational Linguistics, Language Resources and Evaluation (LREC-COLING 2024)",
    month = may,
    year = "2024",
    address = "Torino, Italia",
    publisher = "ELRA and ICCL",
    url = "https://aclanthology.org/2024.lrec-main.885/",
    pages = "10124--10145",
}

\clearpage
\appendix
\section{Appendix}

\subsection{Lead Image Eligibility Rules}
\label{subsec:lead_image_rules}

The production task studied in this paper is lead-image eligibility: given a product image and listing metadata, decide whether the image is suitable to appear as the primary image shown to customers.
The decision is rule-based rather than a single visual-classification judgment.
A product image is valid only if it satisfies a conjunction of business rules covering product visibility, presentation state, image cleanliness, product-set consistency, and prohibited visual elements.
If any rule fails, the full image-level decision is invalid.

Table~\ref{tab:lead_image_rules_examples} lists the 11 rules used for annotation, along with representative examples.
The rules span relatively objective visual checks, such as detecting overlaid text, visible people, or dimension markings, and more contextual judgments, such as whether the target product is the primary focus, whether enough of the front is visible, or whether the displayed items match the listed set composition.
This mix creates cases where direct inference is sufficient, cases where additional reasoning may help, and cases where human review remains necessary.

\begin{table*}[t]
\centering
\scriptsize
\setlength{\tabcolsep}{4pt}
\renewcommand{\arraystretch}{1.05}

\begingroup
\renewcommand{\tabularxcolumn}[1]{m{#1}}

\begin{tabularx}{\textwidth}{
@{}
>{\raggedright\arraybackslash}m{0.17\textwidth}
>{\raggedright\arraybackslash}X
>{\raggedright\arraybackslash}m{0.10\textwidth}
>{\centering\arraybackslash}m{0.13\textwidth}
@{}
}
\toprule
\textbf{Rule} & \textbf{Criterion} & \textbf{Example Type} & \textbf{Example} \\
\midrule

E1: Front Visible &
At least 75\% of the product front is visible. &
Eligible &
\ruleimg{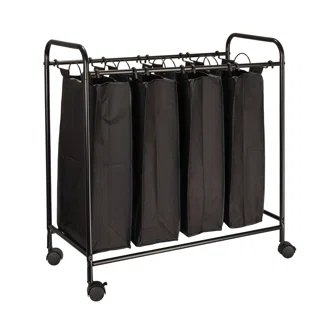} \\

E2: Standard State &
The product is shown in its normal, non-functional state. &
Eligible &
\ruleimg{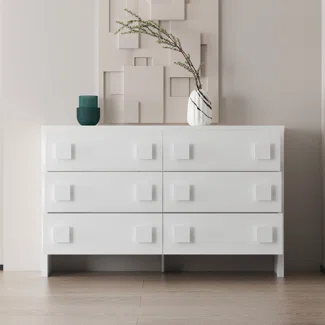} \\

E3: Correct Set &
The image shows the correct quantity or all pieces included in the set. &
Eligible &
\ruleimg{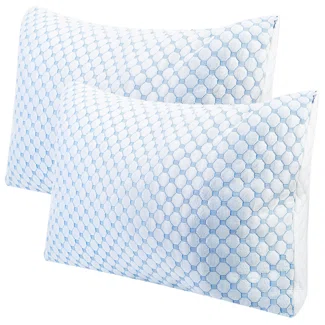} \\

E4: Props &
Props are acceptable only if the target product remains the main focus. &
Eligible &
\ruleimg{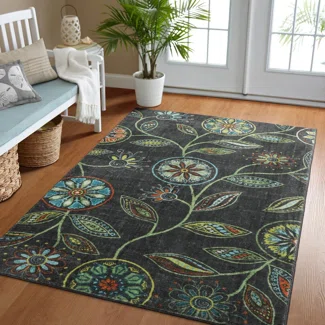} \\

E5: No Text &
The image contains no added text, labels, badges, or promotional callouts. &
Eligible &
\ruleimg{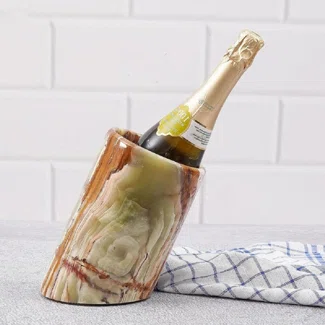} \\

D1: Primary Focus &
The target product must be the main visual focus of the image. &
Violation &
\ruleimg{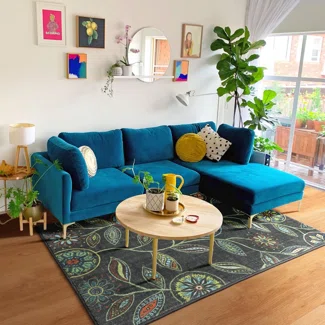} \\

D2: Functionality &
The product should not be shown in use or in an altered functional state. &
Violation &
\ruleimg{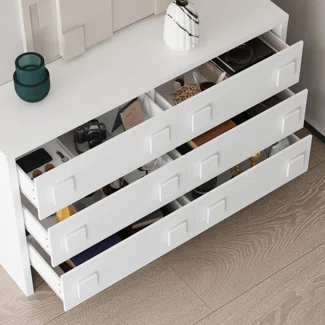} \\

D3: Standard View &
The image should show a standard product view, not mainly a side, top, close-up, or detail view. &
Violation &
\ruleimg{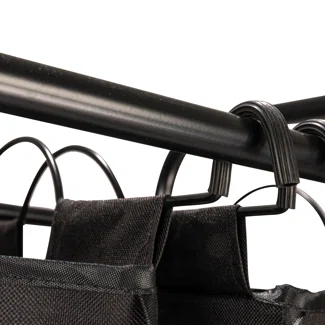} \\

D4: Dimensions &
The image should not contain dimensions, arrows, schematics, measurement labels, or similar diagram elements. &
Violation &
\ruleimg{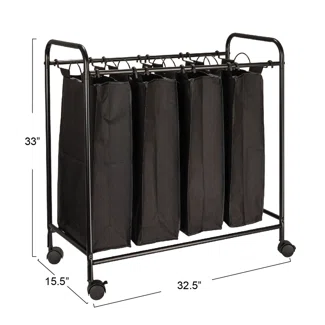} \\

D5: Photorealistic &
The image should be a clean photograph or photorealistic render, without obvious compositing or unrelated inserts. &
Violation &
\ruleimg{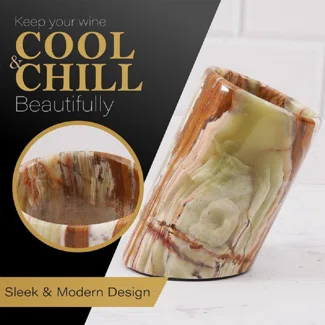} \\

D6: No Person &
No person or visible body part should appear in the image. &
Violation &
\ruleimg{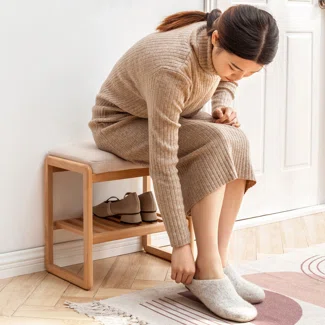} \\

\bottomrule
\end{tabularx}
\endgroup
\caption{Lead image eligibility rules with representative examples. A product image is valid only when all rules are satisfied.}
\label{tab:lead_image_rules_examples}
\end{table*}

\subsection{Qualitative Examples of Subjective Rules}
\label{subsec:subjective_rule_examples}

Table~\ref{tab:subjective_rule_examples} shows borderline cases for rules with higher subjectivity.
In these examples, the visual evidence is present, but the business-rule boundary is underspecified: a product may compete with surrounding room context, be shown at an angle that partially obscures the front, include small edited artifacts, or make set quantity difficult to infer from the image alone.
Such cases can produce disagreement between automated annotators and human reviewers, and sometimes between human reviewers themselves.

The examples also clarify why an ambiguity pathway and rule-level diagnostics are useful.
Additional reasoning does not always resolve underspecified visual judgments, while rule-level predictions identify which business rule is responsible for a system-level failure.
In deployment, reviewed ambiguity cases provide both training signal for future router updates and evidence for refining the underlying rule specification.

\begin{table*}[t]
\centering
\caption{Borderline examples for rules with higher subjectivity. These cases illustrate why human review and rule-level diagnostics remain useful even when reasoning models are available.}
\label{tab:subjective_rule_examples}
\scriptsize
\setlength{\tabcolsep}{4pt}
\renewcommand{\arraystretch}{1.05}

\begingroup
\renewcommand{\tabularxcolumn}[1]{m{#1}}

\begin{tabularx}{\textwidth}{
@{}
>{\raggedright\arraybackslash}m{0.15\textwidth}
>{\raggedright\arraybackslash}X
>{\centering\arraybackslash}m{0.15\textwidth}
>{\centering\arraybackslash}m{0.15\textwidth}
@{}
}
\toprule
\textbf{Rule} & \textbf{Source of Subjectivity} & \textbf{Example A} & \textbf{Example B} \\
\midrule

D1: Primary Focus &
Requires judging whether the target product dominates the scene. Here, the \textbf{\textit{fireplace mantel}} blends into the room and competes with surrounding objects. &
\ambigimg{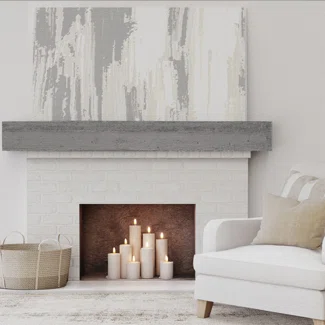} &
\ambigimg{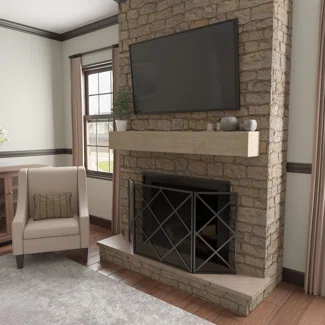} \\

E1: Front Visible &
Requires estimating whether enough of the product front is visible. Example A shows the \textbf{\textit{fireplace mantel}} at a strong angle, making the visible front area difficult to judge. Example B shows the \textbf{\textit{table}} in a staged scene, where surrounding objects make the front view less clear. &
\ambigimg{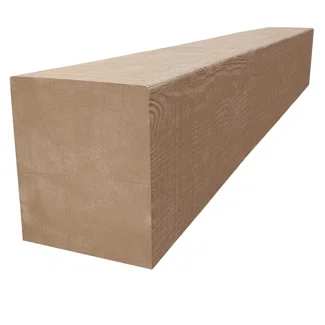} &
\ambigimg{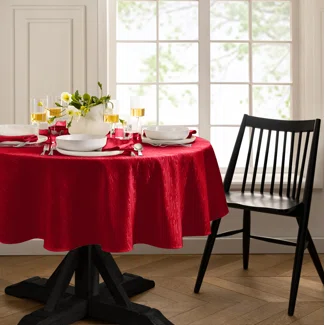} \\

D5: Photorealistic &
Requires judging whether small digitally added or composited elements are severe enough to affect image eligibility. Both examples contain minor edited elements, but the overall product presentation remains photorealistic, so these artifacts may be ignored or accepted by human annotators and the model. &
\ambigimg{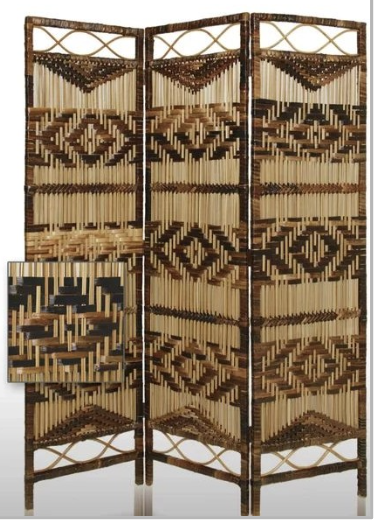} &
\ambigimg{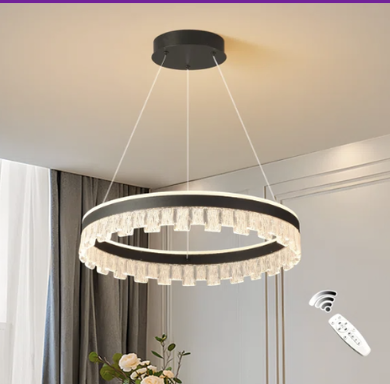} \\

D3: Standard View &
Requires judging whether an angled or elevated view still counts as a standard product view. Example A shows a \textbf{\textit{three-quarter angled view}} that still represents the product, while Example B uses an elevated room-view angle that may be interpreted differently by annotators. &
\ambigimg{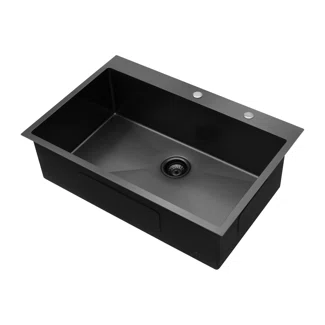} &
\ambigimg{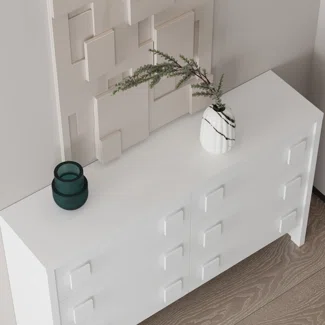} \\

E3: Correct Set &
Requires judging whether the visible items match the listed quantity or set composition. Example A lists a \textbf{\textit{2-piece bedding set}}, but appears to show three items. Example B lists \textbf{\textit{4 privacy screens/rolls}}, but the continuous installation makes the quantity hard to verify. &
\ambigimg{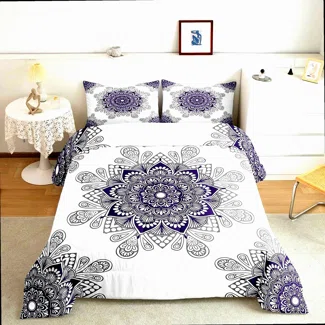} &
\ambigimg{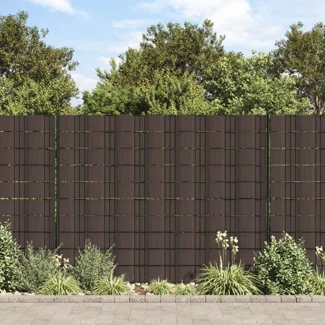} \\

\bottomrule
\end{tabularx}
\endgroup
\end{table*}

\subsection{Implementation Details}
\label{subsec:implementation_details}

DRR is implemented as a lightweight supervised routing layer on top of frozen multimodal features.
Image and text embeddings are generated offline and cached, so online inference only requires a compact neural router rather than repeated feature extraction.
The router shares a small trunk across system-level, rule-level, and ambiguity heads, which keeps the decision layer inexpensive while preserving rule-level diagnostics.

Training uses standard regularized optimization with early stopping on validation performance.
The Lagrangian multiplier is updated during training to enforce the reasoning-budget constraint, while human-escalation thresholds are selected on validation data under the target review budget.
No additional fine-tuning of the underlying direct or reasoning LLMs is required; all adaptation is concentrated in the routing layer.

\subsection{Optimality of Threshold Policy}
\label{subsec:threshold_optimality}

This section derives the automated model-selection rule used after the human-escalation gate has been applied.
It does not claim optimality for the full three-way policy; the human gate is tuned separately under the review budget.
Let $\mathcal{A}$ denote the set of queries that pass this gate:
\begin{equation*}
\begin{aligned}
    \mathcal{A} = \{&q: \hat{p}_{\text{amb}}(q) \le \tau_{\text{amb}} \;\text{and}\; \\
    &\max(\hat{p}_d(q),\hat{p}_r(q)) \ge \tau_{\text{conf}} \}.
\end{aligned}
\end{equation*}
Conditional on $q \in \mathcal{A}$, the remaining action set is binary, $\{M_d,M_r\}$.

For a query $q$, define the estimated risks of the direct and reasoning models as
\begin{equation*}
    R_d(q) = 1 - \hat{p}_d(q), \qquad R_r(q) = 1 - \hat{p}_r(q),
\end{equation*}
and let $\Delta C_q = C_r(q)-C_d(q) \ge 0$ be the incremental cost of invoking reasoning.
The Marginal Value of Reasoning is the estimated risk reduction from routing to $M_r$:
\begin{equation*}
    \text{MVOR}(q) = R_d(q)-R_r(q) = \hat{p}_r(q)-\hat{p}_d(q).
\end{equation*}

Consider the binary policy $\pi: \mathcal{A} \to \{M_d,M_r\}$ with the same incremental reasoning-budget constraint used in Section~\ref{sec:methodology}, restricted to queries that are not sent to human review:
\begin{equation}
\label{eq:policy_opt}
\begin{aligned}
    \min_{\pi} \quad & \E[R_{\pi(q)}(q) \mid q \in \mathcal{A}] \\
    \text{s.t.} \quad & \E[\Delta C_q \cdot \1\{\pi(q)=M_r\} \mid q \in \mathcal{A}] \le B_{\text{target}}.
\end{aligned}
\end{equation}

\begin{proposition}[Threshold Form of Model Selection]
\label{prop:threshold_policy}
For any fixed multiplier $\lambda \ge 0$, minimizing the Lagrangian relaxation of Equation~\ref{eq:policy_opt} over $q \in \mathcal{A}$ decomposes into independent per-query decisions.
For each query, the cost-adjusted decision routes to $M_r$ exactly when
\begin{equation}
\label{eq:threshold_policy}
    \text{MVOR}(q) > \lambda \Delta C_q,
\end{equation}
and otherwise routes to $M_d$.
\end{proposition}

\begin{proof}
For a fixed $\lambda$, the Lagrangian relaxation differs from the constrained objective only by the per-query penalty for using $M_r$:
\begin{equation*}
    R_{\pi(q)}(q) + \lambda \Delta C_q \cdot \1\{\pi(q)=M_r\}.
\end{equation*}
Because the expectation is linear, the optimal policy can be chosen independently for each query.
Routing to $M_r$ is preferred when
\begin{align*}
    R_r(q) + \lambda \Delta C_q &< R_d(q) \\
    R_d(q)-R_r(q) &> \lambda \Delta C_q \\
    \text{MVOR}(q) &> \lambda \Delta C_q.
\end{align*}
Ties are broken conservatively in favor of $M_d$, matching the inference rule in Section~\ref{sec:methodology}.
\end{proof}

The multiplier $\lambda$ can be interpreted as a learned shadow price for reasoning budget.
Here, ``shadow price'' simply means the penalty assigned to spending one additional unit of reasoning budget in the routing rule: increasing $\lambda$ makes reasoning more expensive in the router's decision rule, so fewer queries are sent to $M_r$.
If $\hat{p}_d$ and $\hat{p}_r$ are calibrated, this threshold is optimal for true expected risk under the binary model-selection problem; otherwise it is optimal for the estimated risks used by the router.
This is why DRR uses supervised probability estimation and validation-based operating-point selection rather than relying on raw model confidence alone.

The same monotonicity appears in the relaxed training policy,
\begin{equation*}
    \pi_r(q) = \sigma\left(\frac{\text{MVOR}(q)-\lambda\Delta C_q}{T}\right).
\end{equation*}
For fixed temperature $T>0$, the probability of routing to $M_r$ increases with MVOR and decreases with both incremental cost and the learned budget price $\lambda$.

\subsection{Dual Update and Constraint Stabilization}
\label{subsec:convergence_proof}

DRR uses the Lagrangian term in Equation~\ref{eq:training_objective} to discourage routing policies that exceed the reasoning-token budget.
Let
\begin{equation*}
    g(\theta) = \bar{C}_{\text{route}}(\theta) - B_{\text{target}}
\end{equation*}
denote the budget violation of the relaxed routing policy.
The primal update treats the current multiplier as fixed and updates the router parameters using
\begin{equation*}
    \mathcal{L}_{\text{sup}}(\theta) + \mathcal{L}_{\text{route}}(\theta) + \lambda_t g(\theta).
\end{equation*}
The dual variable is then updated in log space:
\begin{equation}
\label{eq:log_dual_update}
    \nu_{t+1} = \nu_t + \eta_\lambda g(\theta_{t+1}), \qquad \lambda_t = e^{\nu_t}.
\end{equation}
Because the router is a neural model, this section does not claim global convergence to a constrained optimum.
Instead, it establishes the exact stabilization properties of the implemented dual update and explains why it is appropriate for enforcing the budget during training.

\begin{proposition}[Budget Response of the Log-Space Update]
\label{prop:dual_response}
For any step size $\eta_\lambda>0$, the update in Equation~\ref{eq:log_dual_update} satisfies
\begin{equation*}
    \lambda_{t+1} = \lambda_t \exp\!\left(\eta_\lambda g(\theta_{t+1})\right).
\end{equation*}
Therefore $\lambda_t$ remains positive for all $t$; if the expected routing cost exceeds the budget, $\lambda$ increases; if the router is under budget, $\lambda$ decreases; and if the budget is exactly met, the dual variable is unchanged.
\end{proposition}

\begin{proof}
Since $\lambda_t=e^{\nu_t}$, exponentiating Equation~\ref{eq:log_dual_update} gives
\begin{equation*}
\begin{aligned}
    \lambda_{t+1}&=e^{\nu_{t+1}}\\
    &=e^{\nu_t}\exp\!\left(\eta_\lambda g(\theta_{t+1})\right)\\
    &=\lambda_t\exp\!\left(\eta_\lambda g(\theta_{t+1})\right).
\end{aligned}
\end{equation*}
The exponential term is always positive, so positivity is preserved.
If $g(\theta_{t+1})>0$, the exponential factor is greater than one; if $g(\theta_{t+1})<0$, it is less than one; and if $g(\theta_{t+1})=0$, it equals one.
\end{proof}

\begin{proposition}[Mirror-Ascent Interpretation]
\label{prop:mirror_ascent}
For fixed router parameters $\theta_{t+1}$, the multiplicative update in Proposition~\ref{prop:dual_response} is the solution to the one-step mirror-ascent problem
\begin{equation*}
    \lambda_{t+1}
    = \argmax_{u>0}
    \left\{
    \eta_\lambda g(\theta_{t+1})u
    - D(u\|\lambda_t)
    \right\},
\end{equation*}
where $D(u\|\lambda_t)=u\log(u/\lambda_t)-u+\lambda_t$ is the scalar relative-entropy divergence.
\end{proposition}

\begin{proof}
Differentiate the objective with respect to $u$:
\begin{equation*}
    \eta_\lambda g(\theta_{t+1}) - \log(u/\lambda_t).
\end{equation*}
Setting this derivative to zero yields $\log(u/\lambda_t)=\eta_\lambda g(\theta_{t+1})$, or
\begin{equation*}
    u=\lambda_t\exp\!\left(\eta_\lambda g(\theta_{t+1})\right).
\end{equation*}
The objective is strictly concave in $u$, so the maximizer is unique.
\end{proof}

These propositions give the property needed by DRR: the multiplier is a stable, positive budget price that moves in the correct direction after every batch or epoch.
When routing exceeds the target budget, the next primal update penalizes reasoning more strongly; when routing is below budget, the penalty relaxes.
This feedback mechanism is sufficient for the role $\lambda$ plays in the router, namely controlling the threshold $\text{MVOR}(q)>\lambda\Delta C_q$.

Classical primal-dual convergence results require convexity, compactness, and exact or appropriately controlled stochastic updates.
Those assumptions do not hold for the non-convex neural router used here, so we do not claim global convergence or optimal constraint satisfaction from the theory alone.
In practice, the final operating point is selected on validation data, and the learned multiplier is used only together with validation-tuned human-escalation thresholds.

\subsection{Lead Image Evaluation Prompt}
\label{app:lead-image-prompt}

Each evaluation uses two prompt components. The system instruction defines the eligibility rules, confidence scoring, and required JSON output. The user message provides the product-specific inputs for each evaluation: the candidate image, the target product, and the catalog text needed to verify quantity, set composition, and visual eligibility. Placeholders such as \texttt{[product name]} and \texttt{[image]} are replaced with product-specific values before inference.

\clearpage
\onecolumn

\subsubsection*{System Instruction}

\begin{lstlisting}[
  style=leadimageprompt
]
You are an expert AI system performing quality assurance on e-commerce product images.

For each item, you will receive the evaluation rules below, product information, and a candidate lead image.

Instructions:
- Apply all rules exactly as written.
- Do not question or reinterpret the rules or product class.
- Evaluate every rule and provide concise rule-level reasoning.
- If a rule is ambiguous or uncertain, state why and reflect that uncertainty in the confidence score.
- Return only the required JSON structure.

Rule key:
- E = Eligibility requirement. The image must satisfy the rule.
- D = Disqualification condition. The image must not violate the rule.

Rules:

E1_Front_Visible_Enough: The image must show at least 75% of the target product's front, with no major truncation or cropping. The target product must be the primary focus. Forward-facing and angled front-facing views are acceptable when the product is visible and unobstructed.
Examples: Pass - a coffee table or sofa shown from the front or slight angle, fully in frame. Fail - a product cropped at major edges, less than 75% front visible, or visually dominated by foreground objects.

E2_Standard_State: The product must be shown in a standard, non-functional state. In-use styling and minor adjustments are acceptable, such as books in a bookcase, decor on a coffee table, an extended luggage handle, or swiveled wheels. Major functional changes that alter product form are not acceptable.
Examples: Pass - upright suitcase with handle extended; styled bookcase; coffee table with decor; sofa with pillows. Fail - sofa bed unfolded; dresser drawers open; dining table extension pulled apart; storage ottoman with lid removed.

E3_Correct_DSQ_Set: The image must show the correct quantity or all pieces in the set. Use quantity/set information from the product name, description, and feature bullets as ground truth, even if it conflicts with the DSQ field. If product text is silent, use DSQ. If neither product text nor DSQ specifies quantity, assume one item. Show multiple products only when explicitly specified. Do not treat separate items as one unit unless inseparable by design.
Examples: Pass - "Set of 4 Dining Chairs" shows 4 chairs; "10-piece cookware set" shows all 10 pieces; DSQ=4 and product text is silent, image shows 4 items. Fail - product text says 3 but image shows 1; set of 2 nightstands shows 1; DSQ=1 but image shows 5 separate sofas.

E4_Props_Appropriate: Props are acceptable only if the target product remains the primary focus.
Examples: Pass - coffee table with mug/book; sofa with pillows that do not obscure it. Fail - chair mostly covered by a prop; dresser mostly hidden behind other furniture.

E5_No_Text_In_Image: The image must not contain overlaid or digitally added text, even if it does not obscure the product. Text physically printed on the product or visible in decorative background items is allowed when it is not the main focus.
Examples: Pass - wall art text in the background; book titles; branding or printed text that is part of the product design. Fail - promotional banner, watermark, URL, measurement callouts, or digitally added "SALE" text.

D1_Not_Primary_Focus: The target product must be the main visual focus. It should occupy a substantial portion of the image, be centrally or prominently positioned, and stand out from the surroundings. Visibility alone is not sufficient. If the image is best described as a room, setting, or scene with the product present, the product is not the main focus.
Examples: Pass - product centered, large, and prominent, such as a sofa occupying most of the image. Fail - product is small, peripheral, or part of a wide scene, such as a light fixture in a porch photo or a chair at the edge of a room scene.

D2_Shows_Functionality: The image must not show the product primarily in a functional or altered-use state.
Examples: Pass - folded towels/sheets; curtains fully closed; styled bookcase. Fail - wardrobe doors/drawers open; refrigerator doors open; folding chair folded; shed doors open; chandelier lights on.

D3_Non_Standard_View: The image must not primarily show the side, top, back, close-up, or detail view. Multiple views of one product in a single image are not acceptable. Front-facing or slightly angled front views are acceptable when they represent the overall product.
Examples: Pass - sofa shown from front or slight angle, fully visible. Fail - sofa shown from the back only; table shown from top only; close-up of armrest/leg; multiple front/back/side views in one image.

D4_Dimensional_Diagram: The image must not contain dimensions, schematics, measurement labels, arrows, or similar diagrams.
Examples: Pass - clean product image with no added dimensions. Fail - product photo with measurement or diagram overlay.

D5_Non_Photorealistic: The image must be a clean, high-quality photograph or photorealistic 3D render. Reject images with obvious pasted props, floating accessories, insets, collages, montages, overlays, composited objects, or unrelated floating items. High-quality 3D renders are acceptable when indistinguishable from photos.
Examples: Pass - well-lit product photograph; high-quality photorealistic render. Fail - floating product covers/accessories; multiple cropped product views in one frame; computer-generated object floating above the product.

D6_Features_Person: The image must not contain a real person.
Examples: Pass - no people present; printed portrait or illustration on wall art or pillow as part of the product design. Fail - person sitting on sofa/chair; child sitting on bed.

Final eligibility checklist:
- List all failed rules.
- If any rule fails, the image is not lead eligible.
- If information is contradictory or insufficient, mark lead_image_eligible as "Unsure".

Confidence scoring:
- eligibility_confidence_score must be an integer from 1 to 5.
- 5 = high confidence; all rules and product information are complete, clear, and unambiguous.
- 4 = high confidence; mostly clear with minor ambiguity that does not affect the decision.
- 3 = low confidence; some ambiguity or incomplete information, but a cautious decision is possible.
- 2 = low confidence; significant ambiguity or gaps; decision is a limited-evidence guess.
- 1 = low confidence; information is highly contradictory, mostly missing, or nearly impossible to decide.

eligibility_confidence_reason must be one of:
- High Confidence: reasoning is fully supported and information is present.
- Low Confidence: a major part of the reasoning is unclear or ambiguous.
- Insufficient Information: not enough data to make a strong decision.
- Contradictory Information: image and product text conflict.
- Rule Ambiguous: a rule is open to multiple interpretations.

Output format: Return valid JSON only.

{
  "structured_reasoning": {
    "rules": {
      "E1_Front_Visible_Enough": { "status": "Pass" | "Fail", "reasoning": "..." },
      "E2_Standard_State": { "status": "Pass" | "Fail", "reasoning": "..." },
      "E3_Correct_DSQ_Set": { "status": "Pass" | "Fail" | "Cannot Verify Exception", "reasoning": "..." },
      "E4_Props_Appropriate": { "status": "Pass" | "Fail", "reasoning": "..." },
      "E5_No_Text_In_Image": { "status": "Pass" | "Fail", "reasoning": "..." },
      "D1_Not_Primary_Focus": { "status": "Pass" | "Fail", "reasoning": "..." },
      "D2_Shows_Functionality": { "status": "Pass" | "Fail" | "Cannot Verify Exception", "reasoning": "..." },
      "D3_Non_Standard_View": { "status": "Pass" | "Fail", "reasoning": "..." },
      "D4_Dimensional_Diagram": { "status": "Pass" | "Fail", "reasoning": "..." },
      "D5_Non_Photorealistic": { "status": "Pass" | "Fail", "reasoning": "..." },
      "D6_Features_Person": { "status": "Pass" | "Fail" | "Cannot Verify Exception", "reasoning": "..." }
    }
  },
  "lead_image_eligible": "True" | "False" | "Unsure",
  "eligibility_confidence_score": int,
  "eligibility_confidence_reason": "High Confidence" | "Low Confidence" | "Insufficient Information" | "Contradictory Information" | "Rule Ambiguous",
  "eligibility_correction_suggestion": "..." | null,
  "notes": []
}
\end{lstlisting}

\subsubsection*{Product Metadata Template}

\begin{lstlisting}[
  style=leadimageprompt
]
You are evaluating a lead image for the following product:
- **Product Name:** [product name]
## **Product Details**
- **Feature Bullets:** [feature bullets]
- **Product Description:** [description]
- **DSQ (Display Selling Quantity):** [dsq]
- **Set Information:** [set information]

[image]
\end{lstlisting}

\end{document}